\documentclass{article}
\usepackage[accepted]{icml2015}

\usepackage{amsmath,amssymb}
\usepackage{graphicx}
\usepackage{amsthm}
\usepackage{hyperref}
\usepackage{times}

\usepackage{algorithm, algorithmic}

\newtheorem{thm}{Theorem}
\newtheorem{lem}[thm]{Lemma}

\newtheorem*{defn*}{Definition}
\newtheorem*{thm*}{Theorem}

\renewcommand{\>}{{\rightarrow}}

\newcommand{\half}{{\frac {1}{2}}}

\newcommand{\argmax}{\textup{\textrm{argmax}}}
\newcommand{\argmin}{\textup{\textrm{argmin}}}
\newcommand{\R}{{\mathbb R}}
\newcommand{\Z}{{\mathbb Z}}
\newcommand{\N}{{\mathbb N}}

\newcommand{\E}{{\mathbf E}}
\newcommand{\I}{{\mathbf I}}
\newcommand{\X}{{\mathcal X}}
\newcommand{\Y}{{\mathcal Y}}
\newcommand{\T}{{\mathcal T}}

\renewcommand{\L}{{\mathbf L}}

\newcommand{\cN}{{\mathcal N}}

\newcommand{\U}{{\mathcal U}}

\newcommand{\ip}{{\xrightarrow{\textrm{P}}}}

\renewcommand{\S}{{\mathcal S}}

\newcommand{\1}{{\mathbf 1}}
\newcommand{\0}{{\mathbf 0}}

\renewcommand{\b}{{\mathbf b}}
\newcommand{\e}{{\mathbf e}}
\newcommand{\f}{{\mathbf f}}
\newcommand{\g}{{\mathbf g}}

\newcommand{\p}{{\mathbf p}}

\renewcommand{\u}{{\mathbf u}}
\renewcommand{\v}{{\mathbf v}}
\newcommand{\w}{{\mathbf w}}
\newcommand{\x}{{\mathbf x}}

\newcommand{\bell}{{\boldsymbol \ell}}
\newcommand{\balpha}{{\boldsymbol \alpha}}

\newcommand{\bpsi}{{\boldsymbol \psi}}

\newcommand{\sign}{\textup{\textrm{sign}}}

\newcommand{\pred}{\textup{\textrm{pred}}}

\newcommand{\er}{\textup{\textrm{er}}}

\newcommand{\CS}{\textup{\textrm{CS}}}
\newcommand{\BEP}{\textup{\textrm{BEP}}}

\newcommand{\OVA}{\textup{\textrm{OVA}}}

\icmltitlerunning{Consistent Algorithms for Multiclass Classification with a Reject Option }

\begin{document} 

\twocolumn[
\icmltitle{Consistent Algorithms for Multiclass Classification with a Reject Option}

\icmlauthor{Harish G. Ramaswamy}{harish\_gurup@csa.iisc.ernet.in}\vspace{-0.2em}
\icmladdress{Indian Institute of Science, Bangalore, INDIA}\vspace{-0.3em}
\icmlauthor{Ambuj Tewari}{tewaria@umich.edu}\vspace{-0.2em}
\icmladdress{University of Michigan, Ann Arbor, USA}\vspace{-0.3em}
\icmlauthor{Shivani Agarwal}{shivani@csa.iisc.ernet.in}\vspace{-0.2em}
\icmladdress{Indian Institute of Science, Bangalore, INDIA}\vspace{-0.3em}
    
\icmlkeywords{statistical consistency, machine learning, abstain loss, ICML}

\vskip 0.3in
]

\begin{abstract}
We consider the problem of $n$-class classification ($n\geq 2$), where the classifier can choose to abstain from making predictions at a given cost, say, a factor $\alpha$ of the cost of misclassification. Designing consistent algorithms for such $n$-class classification problems with a `reject option' is the main goal of this paper, thereby extending and generalizing previously known results for $n=2$. We show that the Crammer-Singer surrogate and the one vs all hinge loss, albeit with a different predictor than the standard argmax, yield consistent algorithms for this problem when $\alpha=\half$. More interestingly, we design a new convex surrogate  that is also consistent for this problem when $\alpha=\half$ and operates on a much lower dimensional space ($\log(n)$ as opposed to $n$). We also generalize all three surrogates to be consistent for any $\alpha\in[0, \half]$. 
\end{abstract}
\section{Introduction}
In classification problems, one often encounters cases where it would be better for the classifier to take no decision and abstain from predicting rather than making a wrong prediction. For example, in the problem of medical diagnosis with inexpensive tests as features, a conclusive decision is good, but in the face of uncertainty it is better to not make a prediction and go for costlier tests. 

For the case of binary actions, this problem has been called `classification with a reject option' \cite{BarWeg08,YuanWeg10,Grandvalet+08,FumeraRoli02,FumeraRoli04,Fumera+00,Fumera+03,Golfarelli+97}. Yuan and Wegkamp  \yrcite{YuanWeg10} show that many standard convex optimization based procedures for binary classification like logistic regression, least squares classification and exponential loss minimization (Adaboost)  yield consistent algorithms for this problem. But as Bartlett and Wegkamp \yrcite{BarWeg08}  show, the algorithm based on minimizing the hinge loss (SVM) requires a modification to be consistent. The suggested modification is rather simple -- use a double hinge loss with three linear segments instead of the two segments in standard hinge loss, the ratio of slopes of the two non-flat segments depends on the cost of abstaining $\alpha$. 

In the case of multiclass classification however there exist no such results and it is not straightforward to generalize the double hinge loss to this setting. To the best of our knowledge, there has been only empirical and heuristic work on multiclass version of this problem, \cite{Zou+11,Simeone+12,Wu+07}. In this paper, we give a formal treatment of the multiclass problem with a 'reject' option and provide consistent algorithms for this problem.


The reject option is accommodated into the problem of $n$-class classification through the evaluation metric. We now seek a function $h:\X\> \{1,2,\ldots,n,n+1\}$, where $\X$ is the instance space, and the $n$ classes are denoted by $\{1,2,\ldots,n\}=[n]$ and $n+1$ denotes the action of abstaining or the `reject' option. The loss incurred by such a function on an example $(x,y)$ with $h(x)=t$ is given by
\begin{equation}
 \ell^\alpha(y,t)=\begin{cases}
               1 & \text{ if } t\neq y \text{ and } t\neq n+1 \\
               \alpha & \text{ if } t= n+1  \\
               0 & \text{ if } t= y  
              \end{cases}
\label{eqn:abstain-loss}
\end{equation}
where $\alpha\in[0,1]$ denotes the cost of abstaining. We will call this loss the abstain$\left(\alpha\right)$ loss.

It can be easily shown that the Bayes optimal risk for the above loss is attained by the function \mbox{$h^*_\alpha:\X\>\{1,\ldots,n+1\}$} given by
\begin{equation}
 h^*_\alpha(x)=\begin{cases}
               \argmax_{y\in[n]} p_{x}(y) & \text{ if } \max_{y\in[n]} p_{x}(y) \geq 1-\alpha \\
               n+1 & \text{ Otherwise}
               \end{cases}
\label{eqn:Bayes-abstain}
\end{equation}
where $p_{x}(y)=P(Y=y|X=x)$. The above can be seen as a natural extension of the `Chow's rule' \cite{Chow70} for the binary case. It can also be seen that the interesting range of values for $\alpha$ is $[0,\frac{n-1}{n}]$ as for all $\alpha>\frac{n-1}{n}$ the Bayes optimal classifier for the abstain($\alpha$) loss never abstains. For example, in binary classification, only $\alpha\leq \half$ is meaningful, as higher values of $\alpha$ imply it is never optimal to abstain. 

For small $\alpha$, the classifier  $h^*_\alpha$ acts as a high-confidence classifier and would be useful in applications like medical diagnosis. For example, if one wishes to learn a classifier for diagnosing an illness with $80\%$ confidence, and recommend further medical tests if it is not possible, the ideal classifier would be $h^*_{0.2}$, which is the minimizer of the abstain(0.2) loss.  If $\alpha=\half$, the Bayes classifier $h^*_\alpha$ has a very appealing structure -- a class $y\in[n]$ is predicted only if the class $y$ has a simple majority. The abstain($\alpha$) loss is also useful in applications where a `greater than $1-\alpha$ conditional probability detector' can be used as a black box. For example a greater than $\half$ conditional probability detector plays a crucial role in hierarchical classification \cite{Ramaswamy+15}. (Details in supplementary material.) 

Abstain($\alpha$) loss with $\alpha=\half$  will be the main focus of our paper and will be the default when the abstain loss is referred to without any reference to $\alpha$. (As will be the case in Sections \ref{sec:crammer-singer-OVA}, \ref{sec:BEP}, \ref{sec:algo} and \ref{sec:expts}.) 

As it can be seen that the Bayes classifier $h^*_\alpha$ depends only on the conditional distribution of $Y|X$, any algorithm that gives a consistent estimator of the conditional probability of the classes, e.g. minimizing the one vs all squared loss, \cite{RamaswamyAg12,Vernet+11}, can be made into a consistent algorithm (with a suitable change in the decision) for this problem. 


However smooth surrogates that estimate the conditional probability do much more than what is necessary to solve this problem. Consistent piecewise linear surrogate minimizing algorithms, on the other hand do only what is needed and can be expected to be more successful. For example, least squares classification, logistic regression and SVM are all consistent for standard binary classification, but the SVM (which minimizes a piecewise linear hinge loss surrogate) is arguably the most widely used method. Piecewise linear surrogates have other advantages like easier optimization and sparsity (in the dual) as well, hence finding consistent piecewise linear surrogates for the abstain loss is an important and interesting task.

We show that the $n$-dimensional multiclass surrogate of Crammer and Singer \cite{CrammerSi01} and the simple one vs all hinge surrogate loss \cite{Rifkin04} both yield a consistent algorithm for the abstain$\left(\half\right)$ loss. It is interesting to note that both these surrogates are \emph{not} consistent for the standard multiclass classification problem \cite{TewariBa07,Lee+04, Zhang04b}. 

More interestingly, we construct a new convex piecewise linear surrogate, which we call the \emph{binary encoded predictions} (BEP) surrogate that operates on a $\log_2(n)$ dimensional space, and yields a consistent algorithm for the $n$-class abstain$\left(\half\right)$ loss. When optimized over comparable function classes, this algorithm is more efficient than the Crammer-Singer and one vs all algorithms due to requiring to only find $\log_2(n)$ functions over the instance space, as opposed to $n$ functions. This result is surprising because, it has been shown that one needs to minimize at least a $n-1$ dimensional convex surrogate to get a consistent algorithm for the standard $n$-class problem, i.e. \emph{without} the reject option \cite{RamaswamyAg12}. Also the only known generic way of generating consistent surrogate minimizing algorithms for a given loss matrix \cite{RamaswamyAg12}, when applied to the $n$-class abstain loss would give a $n$-dimensional surrogate here.


It is important to note the role of $\alpha$ -- the cost of abstaining. While conditional probability estimation based surrogates can be used for designing consistent algorithms for the $n$-class problem with the reject option with any $\alpha\in(0,\frac{n-1}{n})$, the Crammer-Singer surrogate, the one vs all hinge and the BEP surrogate and their corresponding variants all yield consistent algorithms only for $\alpha\in[0,\half]$. While this may seem restrictive, we contend that  these form an interesting and useful set of problems to solve.
We also suspect that, abstain($\alpha$) problems with $\alpha>\half$ are fundamentally more difficult than those with $\alpha \leq \half$, for the reason that evaluating the Bayes classifier $h^*_\alpha(x)$ can be done for $\alpha \leq \half$ without finding the maximum conditional probability -- just check if any class has conditional probability greater than $(1-\alpha)$ as there can only be one. This is also evidenced by the more complicated partitions of the simplex induced by the Bayes optimal classifier for $\alpha>\half$ as shown in Figure \ref{fig:abstain}.
\vspace{-0.1em}

\subsection{Overview}
\vspace{-0.1em}
We start with some preliminaries and notation in Section \ref{sec:prelims}. In Section \ref{sec:crammer-singer-OVA}  we give excess risk bounds relating the excess Crammer-Singer multiclass surrogate risk and one vs all hinge surrogate risk to the excess abstain$\left(\half\right)$ risk. In Section \ref{sec:BEP} we give our $\log_2(n)$ dimensional BEP surrogate, and give similar excess risk bounds. In Section \ref{sec:algo}, we frame the learning problem with the BEP surrogate as an optimization problem, derive its dual and give a block co-ordinate descent style  algorithm for solving it. In Section \ref{sec:extension} we give generalizations of the Crammer-Singer, one vs all hinge and BEP surrogates that are consistent for abstain$(\alpha)$ loss for $\alpha\in\left[0,\half\right]$. In Section \ref{sec:expts} we include experimental results for all three algorithms. We conclude in Section \ref{sec:concl} with a summary.
\vspace{-0.1em}
\section{Preliminaries}
\label{sec:prelims}
\vspace{-0.1em}
Let the instance space be  $\X$, the finite set of class labels be $\Y=[n]=\{1,\ldots,n\}$, and the finite set of target labels be given by $\T = [n+1] = \{1,\ldots,n+1\}$. Given training examples $(X_1,Y_1),\ldots,(X_m,Y_m)$ drawn i.i.d.\ from a distribution $D$ on $\X\times\Y$, the goal is to learn a prediction model $h:\X\>\T$.   

\begin{figure*}[t]
\vspace{-1em}
\begin{picture}(200,100)
\put(10,20){${\underset{\textrm{\rule{0pt}{11pt}\normalsize{(a)}}}{\left[ \begin{array}{cccc}
 	0 & 1 & 1 & \alpha \\
    1 & 0 & 1 & \alpha \\
    1 & 1 & 0 & \alpha \\
\end{array} \right]}}$}

\put(120,3){${\underset{\textrm{\rule{0pt}{11pt}\normalsize{(b)}}}{\includegraphics[width=100\unitlength]{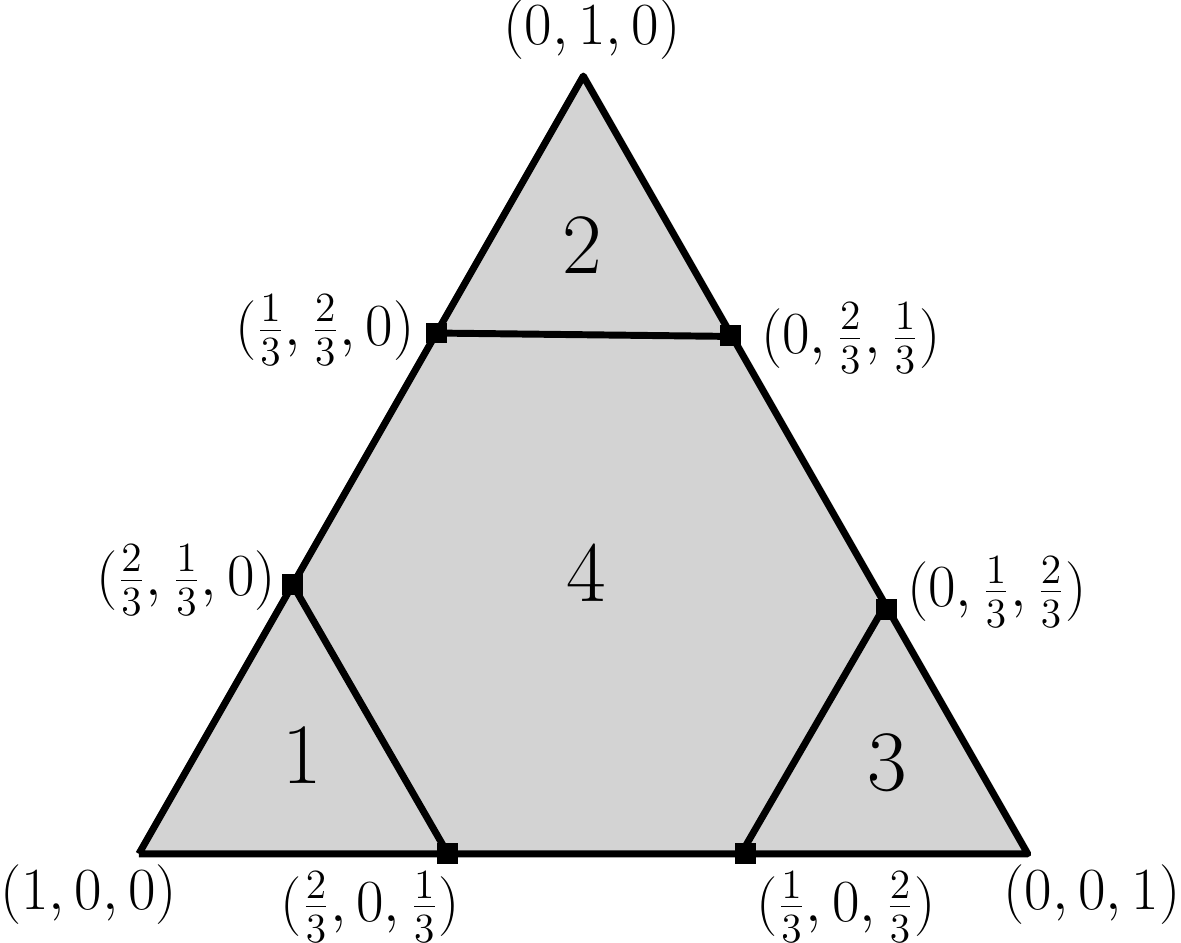}}}$}
\put(250,3){${\underset{\textrm{\rule{0pt}{11pt}\normalsize{(c)}}}{\includegraphics[width=100\unitlength]{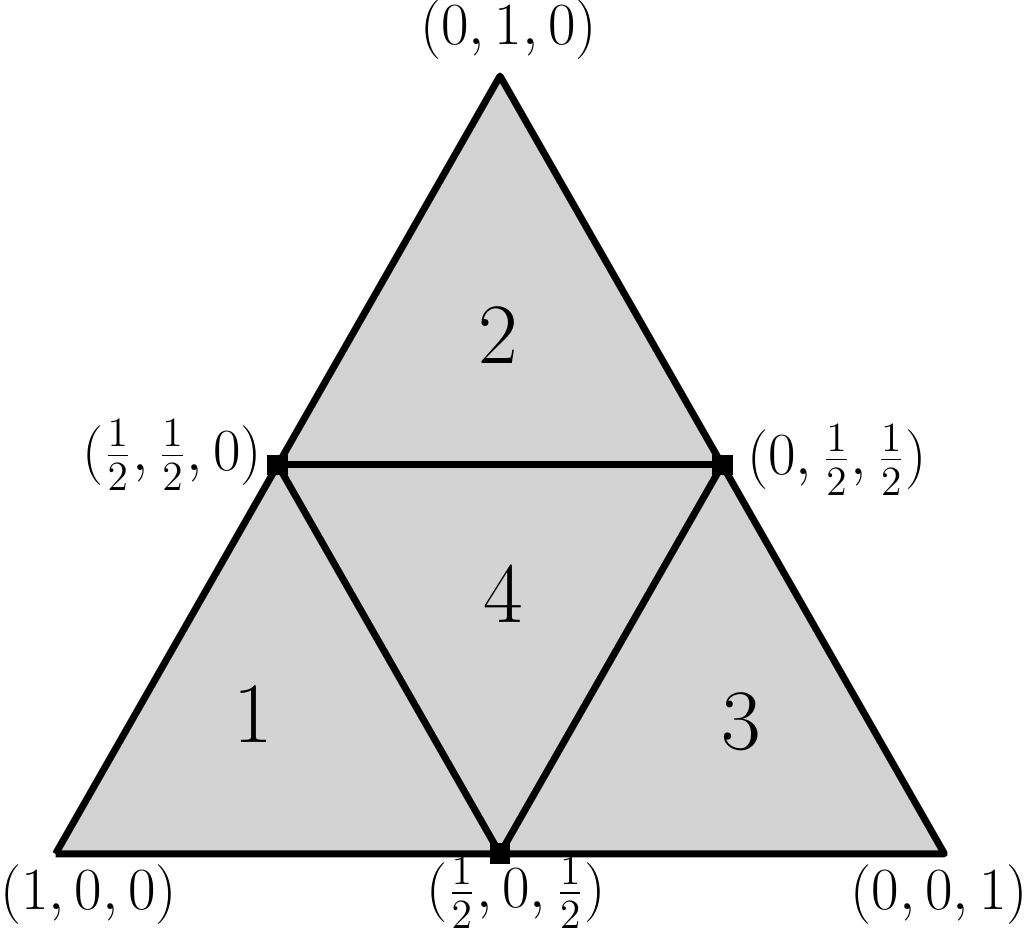}}}$}
\put(380,3){${\underset{\textrm{\rule{0pt}{11pt}\normalsize{(d)}}}{\includegraphics[width=100\unitlength]{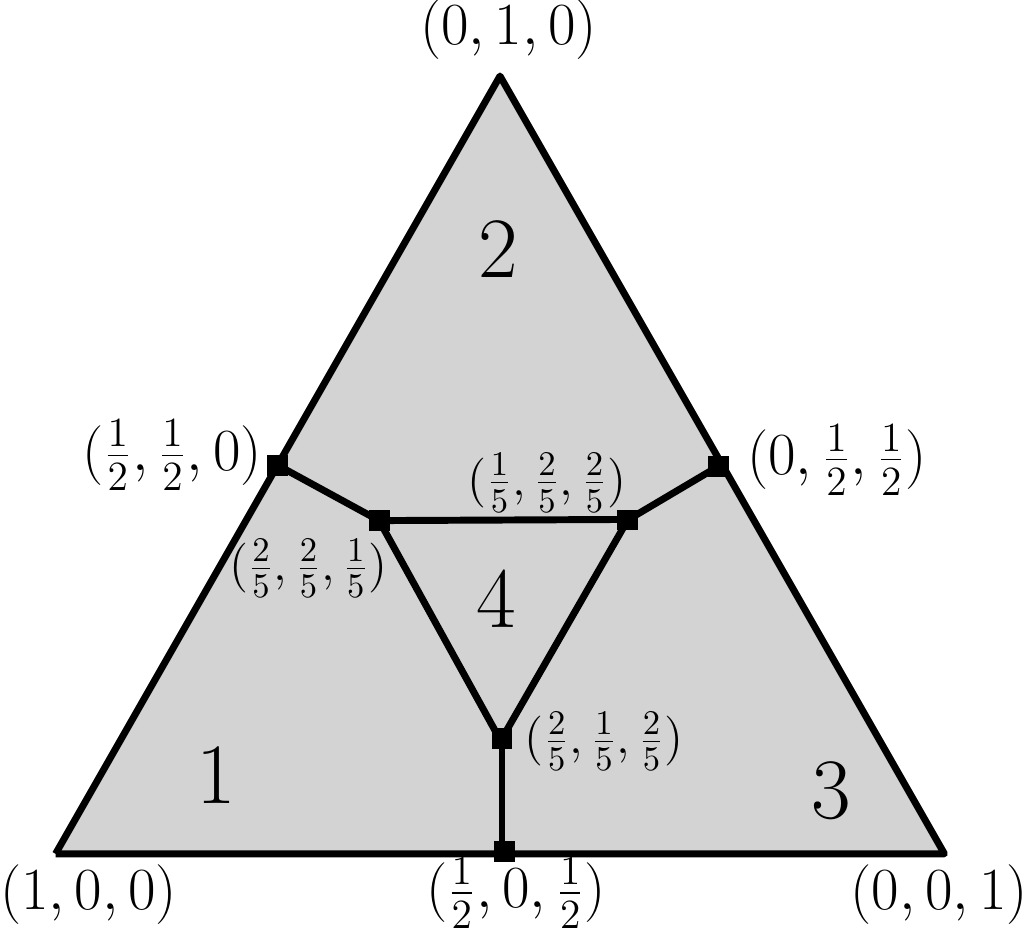}}}$}
\end{picture}
\caption{(a) The abstain$(\alpha)$ loss matrix \mbox{(with n = 3)}; (b,c,d) the partition of the simplex $\Delta_ 3$, depicting the optimal prediction for different conditional probabilities, induced by the Bayes classifier for the abstain($\frac{1}{3}$), abstain($\frac{1}{2}$) and abstain($\frac{3}{5}$) losses respectively.}
\vspace{-1em}
\label{fig:abstain}
\end{figure*}

For any given $\alpha\in[0,1]$, the performance of a prediction model $h:\X\>\T$ is measured via the abstain$(\alpha)$ loss $\ell^\alpha: \Y \times\T\>\R_+$ from Equation (\ref{eqn:abstain-loss}). $\ell^\alpha(y,t)$ denotes the loss incurred on predicting $t$ when the truth is $y$. We will find it convenient to represent the loss function $\ell^\alpha:\Y\times\T\>\R_+$ as a loss matrix $\L^\alpha\in\R_+^{n\times (n+1)}$ with elements $L^\alpha_{yt} = \ell^\alpha(y,t)$ for $y\in[n],t\in[n+1]$, and column vectors $\bell^\alpha_t = (L^\alpha_{1t},\ldots,L^\alpha_{nt})^\top \in\R_+^n$ for $t\in[n+1]$. The abstain($\alpha)$ loss matrix and a schematic representation of the Bayes classifier for various values of $\alpha$ given by equation (\ref{eqn:Bayes-abstain}) are given in Figure \ref{fig:abstain} for $n=3$.

Specifically, the goal is to learn a model \mbox{$h:\X\>\T$} with low expected loss or $\ell^\alpha$-error 
$$\er_D^{\ell^\alpha}[h] = \E_{(X,Y)\sim D}[\ell^\alpha(Y,h(X))]\;.$$ 
Ideally, one wants the $\ell^\alpha$-error of the learned model to be close to the optimal $\ell^\alpha$-error 
$$\er_D^{\ell^\alpha,*} = \inf_{h:\X\>\T}\er_D^{\ell^\alpha}[h]\;.$$

An algorithm, which outputs a (random) model \mbox{$h_m:\X\>\T$} on being given a random training sample as above, is said to be \emph{consistent} w.r.t.\ $\ell^\alpha$ if the \mbox{$\ell^\alpha$-error} of the learned model $h_m$ converges in probability to the optimal for all distributions $D$: $\er_D^{\ell^\alpha}[h_m] \,\ip\, \er_D^{\ell,*}$ . 
Here the convergence in probability is over the learned classifier $h_m$ as a function of the training sample distributed i.i.d. according to $D$.

However, minimizing the discrete $\ell^\alpha$-error directly is computationally difficult; therefore one uses instead a \emph{surrogate loss function} $\psi:\Y\times\R^d\>\overline{\R}_+$ (where $\overline{\R}_+ = [0,\infty]$), for some $d\in\Z_+$, and learns a model $\f:\X\>\R^d$ by minimizing (approximately, based on the training sample) the $\psi$-error 
$$\er_D^\psi[\f] = \E_{(X,Y)\sim D}[\psi(Y,\f(X))]\;.$$ 
Predictions on new instances $x\in\X$ are then made by applying the learned model $\f$ and mapping back to predictions in the target space $\T$ via some mapping $\pred:\R^d\>\T$, giving $h(x) = \pred(\f(x))$.

Under suitable conditions, algorithms that approximately minimize the $\psi$-error based on a training sample are known to be consistent with respect to $\psi$, i.e.\ to converge in probability to the optimal $\psi$-error 
$$\er_D^{\psi,*} = \inf_{\f:\X\>\R^d}\er_D^\psi[\f]\;.$$
Also, when $\psi$ is convex in its second argument, the resulting optimization problem is convex and can be efficiently solved.

Hence, we seek a surrogate and a predictor $(\psi,\pred)$, with $\psi$ convex over its second argument, and satisfying a bound of the following form holding for all $\f:\X\>\R^d$
$$\er_D^{\ell^\alpha}[\pred\circ \f] -\er_D^{\ell^\alpha,*} \leq \xi \left( \er_D^\psi[\f] - \er_D^{\psi,*} \right)$$ 
where $\xi:\R\>\R$ is increasing, continuous at $0$ and $\xi(0)=0$. A surrogate and a predictor $(\psi,\pred)$, satisfying such a bound, known as an excess risk transform bound, would immediately give an algorithm consistent w.r.t. $\ell^\alpha$ from an algorithm consistent w.r.t. $\psi$. We derive such bounds w.r.t. the $\ell^\half$ loss for the Crammer-Singer surrogate, the one vs all hinge surrogate, and the BEP surrogate, with $\xi$ as a linear function.
\vspace{-0.1em}
\section{Excess Risk Bounds for the Crammer-Singer and One vs All Hinge Surrogates}
\label{sec:crammer-singer-OVA}
\vspace{-0.1em}
In this section we give an excess risk bound relating the abstain loss $\ell$, and the Crammer-Singer surrogate  $\psi^\CS$ \cite{CrammerSi01} and also the one vs all Hinge loss.

Define the surrogate $\psi^\CS:[n]\times\R^n\>\R_+$ and predictor $\pred^\CS_\tau:\R^n\>[n+1]$ as
\begin{eqnarray*}
 \psi^\CS(y,\u) &=& (\max_{j\neq y} u_j -u_y + 1)_+ \\
 \pred^\CS_\tau(\u)&=&\begin{cases}
                    \argmax_{i\in[n]} u_i &\text{ if } u_{(1)} - u_{(2)} > \tau \\
                    n+1				&\text{otherwise}
                   \end{cases}
\end{eqnarray*}
where $(a)_+=\max(a,0)$, $u_{(i)}$ is the $i$ th element of the components of $\u$ when sorted in descending order and $\tau\in(0,1)$ is a threshold parameter. 

We proceed further and also define the surrogate and predictor for the one vs all hinge loss.
The surrogate $\psi^\OVA:[n]\times\R^n\>\R_+$ and predictor $\pred^\OVA_\tau:\R^n\>[n+1]$ are defined as
\[
 \psi^\OVA(y,\u) = \sum_{i=1}^n  \1(y=i)(1-  u_i)_+ + \1(y\neq i) (1 + u_i)_+ 
\]
\[
 \pred^\OVA_\tau(\u)=\begin{cases}
                    \argmax_{i\in[n]} u_i &\text{ if } \max_j u_j >  \tau \\
                    n+1				&\text{otherwise}
                   \end{cases}
 \]
where $(a)_+=\max(a,0)$ and $\tau\in(-1,1)$ is a threshold parameter, and ties are broken arbitrarily, say, in favor of the label $y$ with the smaller index. 

The following is the main result of this section, the proof of which is in Appendix \ref{sec:app-A} and \ref{sec:app-B}.
\begin{thm}
\label{thm:CS-OVA-abstain-excess-risk}
Let $n\in\N$ , $\tau_\CS\in(0,1)$ and $\tau_\OVA\in(-1,1)$. Then for all $\f:\X\>\R^n$ 
\vspace{-1em}
$$ $$ 
\begin{eqnarray*} 
\er_D^{\ell }[\pred^\CS_{\tau_\CS} \circ \f] - \er_D^{\ell ,*} 
&\leq&  
\frac{\left(\er_D^{\psi^\CS}[\f] - \er_D^{\psi^\CS,*}\right)}{2 \min(\tau_\CS,1-\tau_\CS)} \\
\er_D^{\ell }[\pred^\OVA_{\tau_\OVA} \circ \f] - \er_D^{\ell ,*}  
&\leq&  
\frac{\left(\er_D^{\psi^\OVA}[\f] - \er_D^{\psi^\OVA,*}\right)}{2 (1-|\tau_\OVA|)}
\end{eqnarray*}
\end{thm}
\vspace{-0.4em}

\textbf{Remark:} It has been pointed out previously by Zhang \yrcite{Zhang04b}, that if the data distribution $D$ is such that $\max_y p_x(y) > 0.5$ for all $x\in\X$, the Crammer-Singer surrogate $\psi^\CS$ and the one vs all hinge loss are consistent with the zero-one loss when used with the standard argmax predictor. Our Theorem \ref{thm:CS-OVA-abstain-excess-risk} implies the above observation. However it also gives more -- in the case that the distribution does not satisfy the dominant class assumption, the model learned by using the surrogate and predictor $(\psi^\CS,\pred^\CS_\tau)$ or $(\psi^\OVA,\pred^\OVA_\tau)$ asymptotically still gives the right answer for instances having a dominant class, and fails in a graceful manner by abstaining for instances that do not have a dominant class.

\section{Excess Risk Bounds for the BEP Surrogate}
\label{sec:BEP}
\vspace{-0.1em}
The Crammer-Singer surrogate and the one vs all hinge surrogate, just like surrogates designed for conditional probability estimation, are defined over an $n$-dimensional domain. Thus any algorithm that minimizes these surrogates must learn $n$ real valued functions over the instance space. In this section, we construct a $\lceil\log_2(n)\rceil$ dimensional convex surrogate, which we call as the \emph{binary encoded predictions} (BEP) surrogate and give an excess risk bound relating this surrogate and the abstain loss. In particular these results show that the BEP surrogate is calibrated w.r.t. the abstain loss; this in turn implies that the \emph{convex calibration dimension} (CC-dimension) \cite{RamaswamyAg12} of the abstain loss is at most $\lceil\log_2(n)\rceil$.

For the purpose of simplicity let us assume $n=2^d$ for some positive integer $d$.\footnote{If $n$ is not a power of $2$, just add enough dummy classes that never occur.} Let $B:[n]\>\{+1,-1\}^d$ be any one-one and onto mapping, with an inverse mapping $B^{-1}:\{+1,-1\}^d\>[n]$. Define the BEP surrogate $\psi^\BEP:[n]\times\R^d\>\R_+$ and its corresponding predictor $\pred^\BEP_\tau:\R^d\>[n+1]$ as\\

\vspace{-1.8em}

\begin{minipage}{0.45 \textwidth}
\begin{eqnarray*}
 \psi^\BEP(y,\u)= (\max_{j\in[d]} B_j(y)u_j + 1)_+
\end{eqnarray*}
\end{minipage}
\begin{minipage}{0.02\textwidth}
 
\end{minipage}
\begin{minipage}[b]{0.53\textwidth}
\begin{equation*}
\pred^\BEP_\tau(\u)=\begin{cases}
                 n+1  & \text{if } \min_{i\in[d]} |u_i| \leq \tau \\
                 B^{-1}(\text{sign}(-\u))  & \text{Otherwise}
                \end{cases}
\end{equation*}
\end{minipage}

where $\sign(u)$ is the sign of $u$, with $\sign(0)=1$ and $\tau\in(0,1)$ is a threshold parameter.

Define the sets $\U^\tau_1,\ldots,\U^\tau_{n+1}$, where
$\U^\tau_k=\{\u\in\R^d:\pred^\BEP_\tau(\u)=k\}$. Which evaluates to
\begin{eqnarray*}
 \U^\tau_y &=& \{\u\in\R^d: \max_j B_j(y) u_j < -\tau \} ~~~ \text{ for }y\in[n]\\ [-0.2em] 
 \U^\tau_{n+1} &=& \{\u\in\R^d: \min_j |u_j| \leq  \tau \} 
\end{eqnarray*}
%
To make the above definition clear we will see how the surrogate and predictor look like for the case of $n=4$ and $\tau=\half$.  We have $d=2$. Let us fix the mapping $B$ such that $B(y)$ is the standard $d$-bit binary representation of $(y-1)$, with $-1$ in the place of $0$. Then we have, 
\begin{eqnarray*}
 \psi^\BEP(1,\u)&=&(\max(-u_1,-u_2) + 1)_+ \\
 \psi^\BEP(2,\u)&=&(\max(-u_1,u_2) + 1)_+\\
 \psi^\BEP(3,\u)&=&(\max(u_1,-u_2) + 1)_+\\
 \psi^\BEP(4,\u)&=&(\max(u_1,u_2) + 1)_+
\end{eqnarray*} 
$$\pred^\BEP_{\half}(\u)=\begin{cases}
                 1 &\text{if }u_1>\half, u_2>\half \\
                 2 &\text{if }u_1>\half, u_2<-\half \\
                 3 &\text{if }u_1<-\half, u_2>\half \\
                 4 &\text{if }u_1<-\half, u_2<-\half \\
                 5 &\text{otherwise}
                \end{cases}$$
\begin{figure}
\begin{center}
 \includegraphics[width=0.4\textwidth]{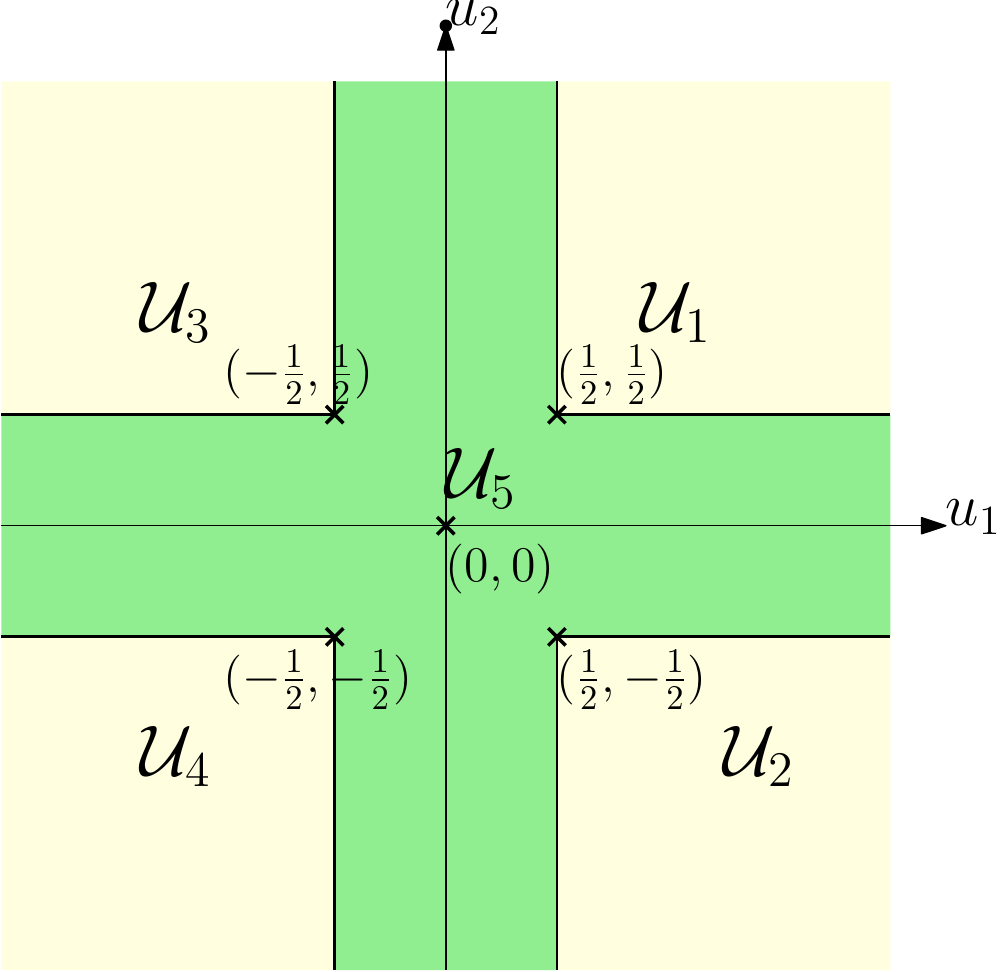}
\end{center}
\caption{The partition of $\R^2$ induced by $\pred^\BEP_\half$}
\label{fig:pred-partition}
\end{figure}
Figure \ref{fig:pred-partition} gives the partition induced by the predictor $\pred^\BEP_\half$.

The following is the main result of this section, the proof of which is in Appendix \ref{sec:app-C}
\begin{thm}
\label{thm:BEP-abstain-excess-risk}
Let $n\in\N$ and $\tau\in(0,1)$. Let $n=2^d$. Then for all $f:\X\>\R^d$
\begin{eqnarray*}
\er_D^{\ell }[\pred^\BEP_\tau \circ \f] - \er_D^{\ell ,*} 
&\leq&  
\frac{\left(\er_D^{\psi^\BEP}[\f] - \er_D^{\psi^\BEP,*}\right) }{2 \min(\tau,1-\tau)}
\end{eqnarray*}
\end{thm}

\textbf{Remark:} The excess risk bounds for the CS, OVA, and BEP surrogates suggest that $\tau=\half$ is the best choice for CS and BEP surrogates, while $\tau=0$ is the best choice for the OVA surrogate. 
However, intuitively $\tau$ is the threshold converting confidence values to predictions, and so it makes sense to use $\tau$ values closer to $0$ (or $-1$ in the case of OVA) to predict aggressively in low-noise situations, and use larger $\tau$ to predict conservatively in noisy situations. Practically, it makes sense to choose the parameter $\tau$ via cross-validation.
\vspace{-0.1em}

\section{BEP Surrogate Optimization Algorithm}
\label{sec:algo}
\vspace{-0.1em}
In this section we frame the problem of finding the linear (vector valued) function that minimizes the BEP surrogate loss over a training set $\{(\x_i,y_i)\}_{i=1}^m$, with $\x_i\in\R^a$ and $y_i\in [n]$, as a convex optimization problem. Once again, for simplicity we assume that the size of the label space $n=2^d$ for some $d\in\Z_+$. The primal and dual of the resulting optimization problem with a norm squared regularizer is given below:

\begin{minipage}[c]{0.485\textwidth}
\textbf{Primal problem:}
\begin{eqnarray*}
 \rule{0ex}{3ex}\min_{\w_1,\ldots,\w_d,\xi_1,\ldots,\xi_m} \sum_{i=1}^m \xi_i + \frac{\lambda}{2} \sum_{j=1}^d ||\w_j||^2 \\[0.4em]
 \text{such that} \hspace{1em} \forall i\in[m],j\in[d] \hspace{4em} \\[0.4em]
 \xi_i \geq B_j(y_i) \w_j^\top \x_i +1 \hspace{3.6em} \\
 \rule{0ex}{3ex}\xi_i \geq 0  \hspace{10em}
\end{eqnarray*}
\end{minipage}

%
\begin{minipage}[c]{0.485\textwidth}
\textbf{Dual problem:}
 \begin{eqnarray*}
\max_{\balpha \in \R^{m\times(d+1)}}   -\sum_{i=1}^m \alpha_{i,0} - \frac{1}{2\lambda} \sum_{i=1}^m \sum_{i'=1}^m \langle \x_i, \x_{i'} \rangle \mu_{i,i'}(\balpha)   \\[0.4em]
\text{such that} \hspace{1em} \forall i\in[m],j\in[d]\cup\{0\} \hspace{7em} \\[0.4em]
\alpha_{i,j} \geq 0  ;\hspace{1.7em}
\sum_{j'=0}^d \alpha_{i,j'} = 1\;. \hspace{6em}   
\end{eqnarray*}
\end{minipage}
where $\mu_{i,i'}(\balpha)=\sum_{j=1}^d B_j(y_i) B_j(y_{i'}) \alpha_{ij}\alpha_{i',j}$. 

\vspace{1em} 
We optimize the dual as it can be easily extended to work with kernels. The structure of the constraints in the dual lends itself easily to a block co-ordinate ascent algorithm, where we optimize over $\{\alpha_{i,j}:j\in\{0,\ldots,d\} \}$ and fix every other variable in each iteration. Such methods have been recently proven to have exponential convergence rate for SVM-type problems \cite{WangLin13}, and we expect results of those type to apply to our problem as well. 

The problem to be solved at every iteration reduces to a $l_2$ projection of a vector $\g^i\in\R^d$ on to the set $\S_i=\{\g\in\R^d:\g^\top \b^i \leq 1\}$, where $\b^i\in\{\pm 1\}^d$ is such that $b^i_j=B_j(y_i)$. The projection problem is a simple variant of projecting a vector on the $l_1$ ball of radius $1$, which can be solved efficiently in $O(d)$ time \cite{Duchi+08}. The vector $\g^i$ is such that for any $j\in[d]$,
\begin{equation*}
 g^i_j=\frac{\lambda}{\langle \x_i,\x_i \rangle} \left(\b^i_j - \frac{1}{\lambda}\left(\sum_{i'=1; i'\neq i}^m \langle \x_i,\x_{i'} \rangle \alpha_{i',j} B_j(y_{i'}) \right) \right).
\end{equation*}

\vspace{-0.1em}
\section{Abstain($\alpha$) Loss for $\alpha<\half$}
\label{sec:extension}
\vspace{-0.1em}
The excess risk bounds derived for the CS, OVA hinge loss and BEP surrogates apply only to the abstain$\left(\half\right)$ loss. But it is possible to derive such excess risk bounds for abstain($\alpha$) with $\alpha\in\left[0,\half \right]$ with slight modifications to the CS, OVA and BEP surrogates.

Define $\psi^{\CS,\alpha}:[n]\times \R^n\>\R_+$, $\psi^{\OVA,\alpha}:[n]\times \R^n\>\R_+$ and $\psi^{\BEP,\alpha}:[n]\times \R^d\>\R_+$, with $n=2^d$ as

\begin{eqnarray*}
 \psi^{\CS,\alpha}(y,\u) &=& 2\cdot\max\bigg(\alpha\max_{j\neq y} \gamma(u_j -u_y), \\[-0.5em]
			  &&(1-\alpha)\max_{j\neq y} \gamma(u_j -u_y)  \bigg)+2\alpha \\
 \psi^{\OVA,\alpha}(y,\u) &=& 2\cdot\bigg(\sum_{i=1}^n\bigg(  \1(y=i)\alpha(1-  u_i)_+ \\[-0.5em]
			  &&\hspace{1em}+~ \1(y\neq i)(1-\alpha) (1 + u_i)_+ \bigg)\bigg)  
\end{eqnarray*}
\begin{eqnarray*}
 \psi^{\BEP,\alpha}(y,\u)&=& 2\cdot\max\bigg(\alpha\max_{j\in[d]} \gamma(B_j(y)u_j), \\[-0.5em]
			  && \hspace{1em}(1-\alpha)\max_{j\in[d]} \gamma(B_j(y)u_j) \bigg) + 2\alpha
 \end{eqnarray*}
where, $\gamma(a) = \max(a,-1)$ and $B:[n]\>\{-1,1\}^d$ is any bijection. Note that $\psi^{\CS,\half}=\psi^{\CS}$, $\psi^{\OVA,\half}=\psi^{\OVA}$ and $\psi^{\BEP,\half}=\psi^{\BEP}$.

One can show the following theorem which is a generalization of Theorems \ref{thm:CS-OVA-abstain-excess-risk} and \ref{thm:BEP-abstain-excess-risk}. The proof proceeds along the same lines as the proofs of Theorems \ref{thm:CS-OVA-abstain-excess-risk} and \ref{thm:BEP-abstain-excess-risk} and is hence omitted.

\begin{thm}
\label{thm:BEP-abstain-alpha-excess-risk}
 Let $n\in\N,\tau\in(0,1),\tau'\in(-1,1)$ and $\alpha\in\left[0,\half\right]$. Let $n=2^d$. Then for all $\f:\X\>\R^d$, $\g:\X\>\R^n$
\begin{eqnarray*}
\lefteqn{\er_D^{\ell^\alpha}[\pred^\CS_\tau \circ \g] - \er_D^{\ell^\alpha,*} }\\
&\leq&  
\frac{1}{2 \min(\tau,1-\tau)}\left(\er_D^{\psi^{\CS,\alpha}}[\g] - \er_D^{\psi^{\CS,\alpha},*}\right) \\
\lefteqn{\er_D^{\ell^\alpha}[\pred^\OVA_{\tau'} \circ \g] - \er_D^{\ell^\alpha,*} }\\
&\leq&  
\frac{1}{2 (1-|\tau'|)}\left(\er_D^{\psi^{\OVA,\alpha}}[\g] - \er_D^{\psi^{\OVA,\alpha},*}\right) \\
\lefteqn{\er_D^{\ell^\alpha}[\pred^\BEP_\tau \circ \f] - \er_D^{\ell^\alpha,*} }\\
&\leq&  
\frac{1}{2 \min(\tau,1-\tau)}\left(\er_D^{\psi^{\BEP,\alpha}}[\f] - \er_D^{\psi^{\BEP,\alpha},*}\right)
\end{eqnarray*}
\end{thm}

\textbf{Remark:}
When $n=2$, the Crammer-Singer surrogate, the one vs all hinge and the BEP surrogate all reduce to the hinge loss and  $\alpha$ is restricted to be at most $\half$ to ensure the relevance of the abstain option. Applying the above extension for $\alpha\leq\half$ to the hinge loss, we get the `generalized hinge loss' of Bartlett and Wegkamp \yrcite{BarWeg08}.
\vspace{-0.1em}

\section{Experimental Results}
\label{sec:expts}
\vspace{-0.1em}
In this section give our experimental results for the algorithms proposed on both synthetic and real datasets. 

\subsection{Synthetic Data}
We optimize the Crammer-Singer surrogate, the one vs all hinge surrogate and the BEP surrogate, over appropriate kernel spaces  on a 2-dimensional 8 class synthetic data set and show that the the abstain$\left(\half\right)$ loss incurred by the trained model for all three algorithms approaches the Bayes optimal under various thresholds.

The dataset we used was generated as follows. We randomly sample 8 prototype vectors $\v_1,\ldots,\v_8 \in \R^2$, with each $\v_y$ drawn independently from a zero mean unit variance 2D-Gaussian, $\cN(\0,\I_2)$ distribution. These 8 prototype vectors correspond to the 8 classes. Each example $(\x,y)$ is generated by first picking $y$ from one of the 8 classes uniformly at random, and the instance $\x$ is set as $\x=\v_y + 0.65\cdot\u$, where $\u$ is independently drawn from $\cN(\0,\I_2)$. We generated 12800 such $(\x,y)$ pairs for training, and another 10000 instances, for testing.

The CS, OVA, BEP surrogates were all optimized over a reproducing kernel Hilbert Space (RKHS) with a Gaussian kernel   and the standard norm-squared regularizer. The kernel width parameter and the regularization parameter were chosen by grid search using a separate validation set.\footnote{We used Joachims' SVM-light package \cite{Joachims99}  for the OVA and CS algorithms.}

As Figure \ref{fig:expts} indicates, the expected abstain risk incurred by the trained model approaches the Bayes risk with increasing training data for all three algorithms and intermediate $\tau$ values. The excess risk bounds in Theorems \ref{thm:CS-OVA-abstain-excess-risk} and \ref{thm:BEP-abstain-excess-risk} breakdown when the threshold parameter $\tau\in\{0,1\}$ for the CS and BEP surrogates, and when $\tau \in\{-1,1\}$ for the OVA surrogate. This is supported by the observation that, in Figure \ref{fig:expts} the curves corresponding to these thresholds perform poorly. In particular, using $\tau=0$ for the CS and BEP algorithms implies that the resulting algorithms never abstain.

Though all three surrogate minimizing algorithms we consider are consistent w.r.t. abstain loss, we find that the BEP and OVA algorithms use less computation time and samples than the CS algorithm to attain the same error.  However, the BEP surrogate performs  poorly when optimized over a linear function class (experiments not shown here), due to its much restricted representation power. 	
\begin{figure*}[!t]
\begin{center}
 $\underset{\textrm{\rule{0em}{2em}\normalsize(a)}}{\includegraphics[width=0.31\textwidth]{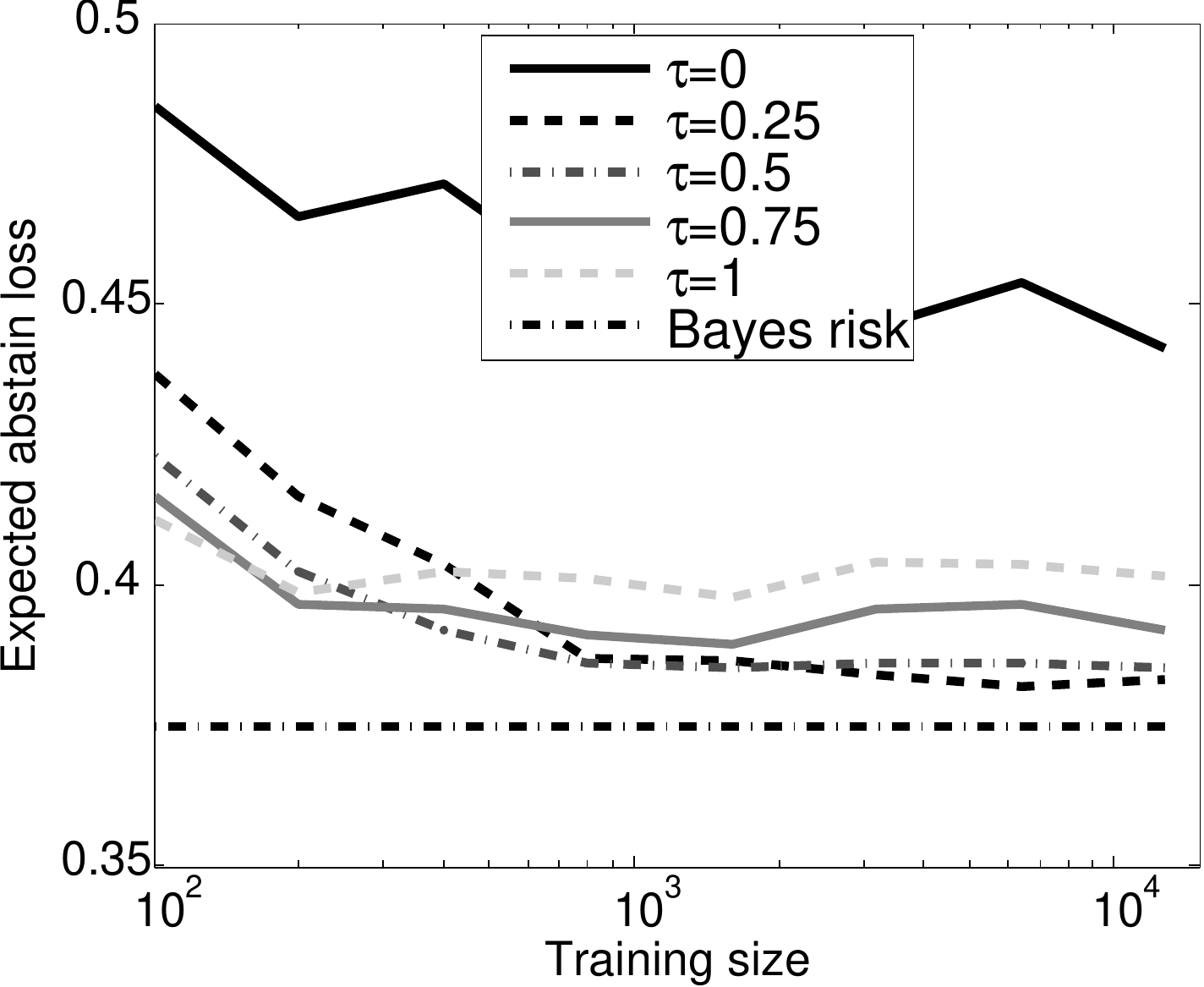}}$
 $\underset{\textrm{\rule{0em}{2em}\normalsize(b)}}{\includegraphics[width=0.31\textwidth]{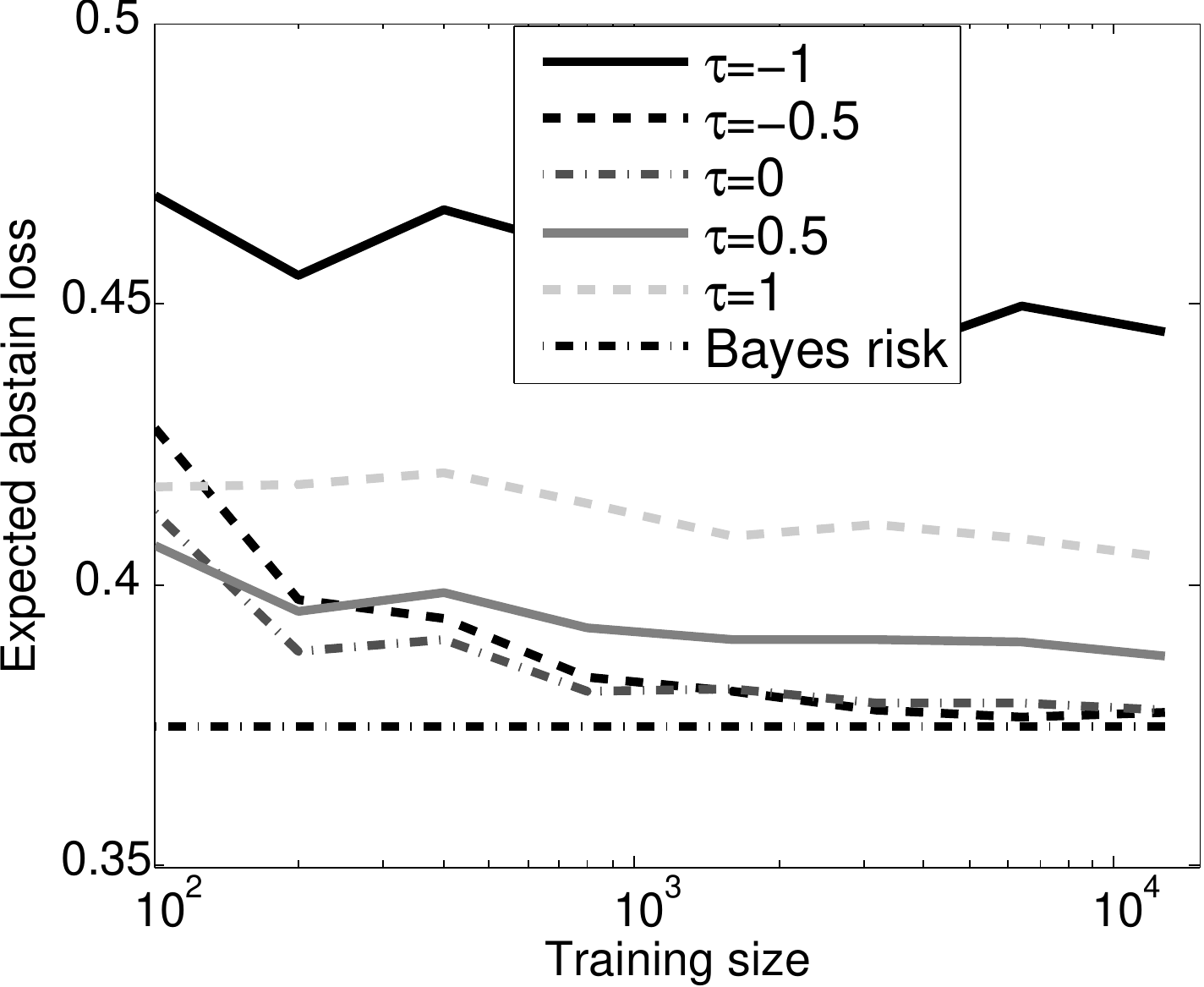}}$
 $\underset{\textrm{\rule{0em}{2em}\normalsize(c)}}{\includegraphics[width=0.31\textwidth]{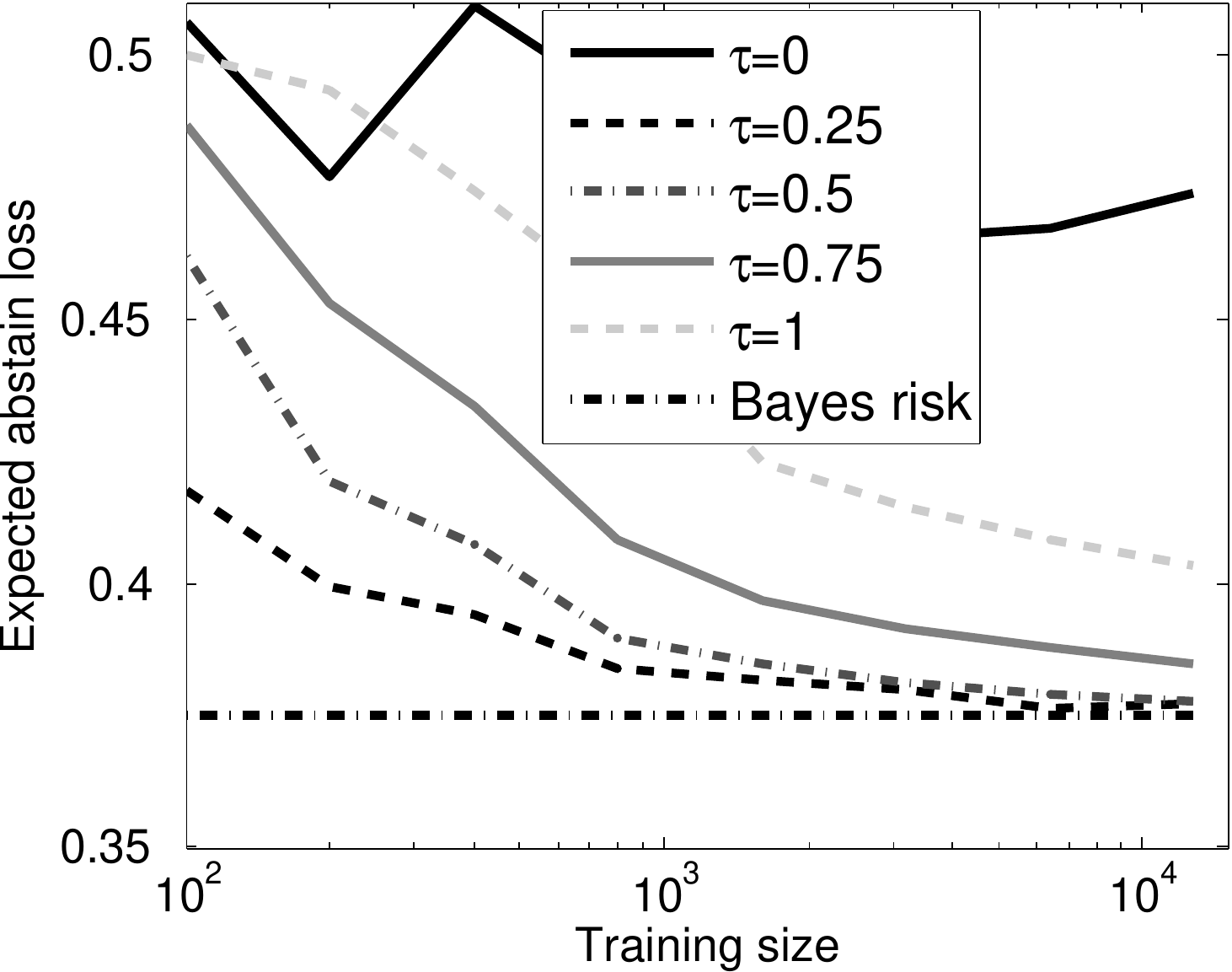}}$
 \end{center}
\vspace{-0.3em}
 \caption{
 (a) Performance of the CS surrogate for various thresholds as a function of training size 
 (b) Performance of the OVA surrogate for various thresholds as a function of training size
 (c) Performance of the BEP surrogate for various thresholds as a function of training size} 

 \label{fig:expts}
\end{figure*}

\subsection{Real Data}

\begin{table*}[!t]
\vspace{-1em}
 \label{tab:results}
 \caption{Error percentages of the three algorithms when the rejection percentage is fixed at 0\%, 20\% and 40\%.}
\begin{center}
 \begin{tabular}{|c||ccc||ccc||ccc|}
  \hline
Reject:			&& 0\% &&& 20\% &&& 40\% &		
\\ \hline \hline
Algorithm:		&CS & OVA & BEP &CS & OVA & BEP &CS & OVA & BEP 
\\ \hline \hline
\texttt{satimage} 	&10.25 & 8.3 & 8.15 
			&5.6   & 2.5 &  2.4
			&2.9   & 0.9 & 0.6
\\ \hline	
  \texttt{yeast}	&44.4 & 38.8 & 42.7
			&34.5 & 26   & 29.7
			&24   & 17   & 19.8
\\ \hline	
  \texttt{letter}	& 4.8 & 2.8 & 4.6
			& 1.4 & 0.1 & 0.6
			& 0.4 & 0   & 0.1
\\ \hline	
  \texttt{vehicle}	& 31.5 & 17.1 & 20.5
			& 24.6 & 8.2 & 13
			& 16.4 & 5.5 & 6.1
\\ \hline	
  \texttt{image}	&5.8 & 5.1 & 4.2
			& 2.2 & 1.6 & 1.6
			& 0.6 & 0.6 & 0.3
\\ \hline	
  \texttt{covertype}	& 32.2 & 28.1 & 29.4
			& 23.6 & 19.3 & 20.4
			& 16.3 & 11.7 & 12.8
\\ \hline
 \end{tabular}
\end{center}
\end{table*}

We ran experiments on real multiclass datasets from the UCI repository, the details of which are in Table \ref{tab:datasets}. In each of these datasets if a train/test split is not indicated in the dataset we make one ourselves by splitting at random. 

\begin{table}[!h]
\caption{Details of datasets used.}
\label{tab:datasets}
 \begin{center}
  \begin{tabular}{|c|c|c|c|c|}
  \hline
   &\# Train & \# Test & \# Feat & \# Class \\ \hline
\texttt{satimage} &4,435	&2,000	&36	&6	\\  
\texttt{yeast} 	&1,000	&484	&8	&10	\\  	
\texttt{letter}	&16,000	&4,000	&16	&26	\\  
\texttt{vehicle}&700	&146	&18	&4	\\  
\texttt{image}	&2,000	&310	&19	&7	\\  
\texttt{covertype} &15,120&565,892&54	&7	\\ \hline  
  \end{tabular}
 \end{center}
\end{table}

All three algorithms (CS, OVA and BEP) were optimized over an RKHS with a Gaussian kernel and the standard norm-squared regularizer. The kernel width and regularization parameters were chosen through validation -- 10-fold cross-validation in the case of \texttt{satimage}, \texttt{yeast}, \texttt{vehicle} and \texttt{image} datasets, and a 75-25 split of the train set into train and validation for the \texttt{letter} and \texttt{covertype} datasets. For simplicity we set $\tau=0$ (or $\tau=-1$ for OVA)  during the validation phase.

The results of the experiment with the CS, OVA and BEP algorithms is given in Table 2. The rejection rate is fixed at  some given level by choosing the threshold $\tau$ for each algorithm and dataset appropriately. As can be seen from the Table, the BEP algorithm's performance is comparable to the OVA, and is better than the CS algorithm. However, Table \ref{tab:times-taken}, which gives the training times for the algorithms, reveals that the BEP algorithm runs the fastest, thus making the BEP algorithm a good option for large datasets. The main reason for the observed speedup of the BEP is that it learns only $\log_2(n)$ functions for a $n$-class problem and hence the speedup factor of the BEP over the OVA would potentially be better for larger $n$. 

\begin{table}
\caption{Time taken for learning final model and making predictions on test set (does not include validation time)}
\label{tab:times-taken}
 \begin{center}
 \begin{tabular}{|c||c|c|c||}
\hline
Algorithm	& CS  & OVA & BEP
\\ \hline \hline	
  \texttt{satimage} 	& 2153s & 76s & 44s
\\ \hline	
  \texttt{yeast}	& 5s & 7s & 2s
\\ \hline	
  \texttt{letter}	&9608s & 1055s & 313s
\\ \hline	
  \texttt{vehicle}	& 3s & 3s & 1s
\\ \hline	
  \texttt{image}	&222s & 16s & 6s
\\ \hline	
  \texttt{covertype}	&47974s & 23709s & 6786s		
\\ \hline	
\end{tabular}
\end{center}
\end{table}

\vspace{-0.5em}
\section{Conclusion}
\label{sec:concl}
\vspace{-0.5em}
The multiclass classification problem with a reject option, is a powerful abstraction that captures controlling the uncertainty of the classifier and is very useful in applications like medical diagnosis.
We formalized this problem via an evaluation metric, called the abstain loss,  and gave excess risk bounds relating the abstain loss to the Crammer-Singer surrogate, the one vs all hinge surrogate and also to the BEP surrogate which is a new surrogate and operates on a much smaller dimension.  Extending these results for other such evaluation metrics, in particular the abstain$(\alpha)$ loss for $\alpha>\half$, is an interesting future direction.




\bibliographystyle{icml2015}
\bibliography{Abstain_loss}

\newpage

\appendix
\twocolumn[
\begin{center}


\textbf{\Large Appendix} \\ 
\rule{50em}{0.1em}
\vspace{1em}
\end{center}
]

We break the Proof of Theorem \ref{thm:CS-OVA-abstain-excess-risk} into two parts consisting of the proof of excess risk bounds for the CS surrogate and the OVA surrogate respectively.
 \section{Proof of Excess Risk Bounds for the Crammer Singer Surrogate}
 \label{sec:app-A}
 Define the sets $\U_1,\ldots,\U_{n+1}$ such that $\U_i$ is the set of vectors $\u$ in $\R^n$, for which $\pred^\CS_\tau(\u)=i$
\begin{eqnarray*}
 \U^\tau_y &=& \{\u\in\R^n: u_y > u_j + \tau \text{ for all }j\neq y\}; ~~y\in[n]\\
 \U^\tau_{n+1} &=& \{\u\in\R^n: u_{(1)} \leq u_{(2)} + \tau \}.
\end{eqnarray*}

The following lemma gives some crucial, but straightforward to prove, (in)equalities satisfied by the Crammer-Singer surrogate.
\begin{lem}
 \label{lem:CS-identitites}
 \begin{eqnarray}
\forall y\in[n],  \forall \p\in\Delta_n &&  \nonumber\\
  \p^\top \bpsi^\CS(\e_y) &=& 2(1 - p_y) ,\hspace{3.2em} \label{eqn:CS-lem-1}\\
  \p^\top \bpsi^\CS(\0) &=& 1 ,\hspace{6.7em}  \label{eqn:CS-lem-2} \\
\forall \u\in\R^n, \forall y \in \argmax_{i} u_i\hspace{-1em} &,& \hspace{-1em}\forall y' \notin \argmax_{i} u_i \nonumber \\
  \psi^\CS(y,\u) &\geq& u_{(2)} - u_{(1)} + 1 ,\hspace{1em}   \label{eqn:CS-lem-3} \\
  \psi^\CS(y',\u) &\geq& u_{(1)} - u_{(2)} + 1  ,\hspace{1em}   \label{eqn:CS-lem-4}
\end{eqnarray}
where $\e_y$ is the vector in $\R^n$ with $1$ in the $y^{th}$ position and $0$ everywhere else.
\end{lem}

The part of Theorem \ref{thm:CS-OVA-abstain-excess-risk} proved here is restated below.
\begin{thm*}
Let $n\in\N$ and $\tau\in(0,1)$. Then for all $\f:\X\>\R^n$ 
\vspace{-1em}
$$\er_D^{\ell }[\pred^\CS_\tau \circ \f] - \er_D^{\ell ,*} \leq  
\frac{\left(\er_D^{\psi^\CS}[\f] - \er_D^{\psi^\CS,*}\right)}{2 \min(\tau,1-\tau)} $$ 
\end{thm*}
\begin{proof}
We will show that $\forall \p\in\Delta_n$ and all $\u\in\R^d$
\begin{eqnarray}
 \lefteqn{\p^\top \bpsi^\CS(\u) - \inf_{u'\in\R^n} \p^\top \bpsi(\u') } \nonumber \\
 &\geq& 2 \min(\tau,1-\tau) \big( \p^\top \bell_{\pred^\CS_\tau(\u)} - \min_t \p^\top \bell_t \big) \;.
 \label{eqn:CS-excess-risk}
\end{eqnarray}
The Theorem simply follows from linearity of expectation. 

\textbf{Case 1:} $p_y\geq \half$ for some $y\in[n]$.
 
We have that $y \in \argmin_t \p^\top \bell _t $.

\textbf{Case 1a:} $\u\in\U^\tau_y$ 

The RHS of equation (\ref{eqn:CS-excess-risk}) is zero, and hence becomes trivial. 

\textbf{Case 1b:} $\u\in\U^\tau_{n+1}$ 

We have that $u_{(1)} - u_{(2)} \leq \tau$.\\
Let $q=\sum_{i \in \argmax_j u_j} p_i$. We then have
\begin{eqnarray}
 \lefteqn{\p^\top \bpsi^\CS(\u) - \p^\top \bpsi^\CS(\e_y) } \nonumber \\
 &\stackrel{(\ref{eqn:CS-lem-1})}{=}&
  \sum_{i: u_i=u_{(1)}} p_i \psi^\CS(i,\u) + \sum_{i: u_i<u_{(1)}}  p_y \psi^\CS(y,\u)   \nonumber \\  
  && \hspace{10em} -2(1-p_y) \nonumber \\
 &\stackrel{(\ref{eqn:CS-lem-3}), (\ref{eqn:CS-lem-4})}{\geq}&
 (2q-1)(u_{(2)} - u_{(1)}) -1 + 2p_y \nonumber \\
 &\geq &
 (2p_y-1)(1-\tau) \hspace{10em}\;. \label{eqn:CS-case-1b-psi}
\end{eqnarray}
The last inequality follows from $u_{(2)}-u_{(1)} \geq -\tau$  and the following observations. If $q>p_y$ then $u_{(1)}=u_{(2)}$, and if $q< p_y$ we have $q< \half$.
\begin{equation}
\label{eqn:CS-case-1b-ell}
 \p^\top \bell_{\pred^\CS_\tau(\u)} - \min_t \p^\top \bell_t = \p^\top \bell_{n+1} - \p^\top \bell_y = p_y - \half
\end{equation}

From Equations (\ref{eqn:CS-case-1b-psi}) and (\ref{eqn:CS-case-1b-ell}) we have
\begin{eqnarray}
\lefteqn{ \p^\top \bpsi^\CS(\u) - \inf_{u'\in\R^n} \p^\top \bpsi(\u') } \nonumber \\ 
&\geq& 
2  (1-\tau) \big( \p^\top \bell_{\pred^\CS_\tau(\u)} - \min_t \p^\top \bell_t \big) \label{eqn:CS-case1b}
\end{eqnarray}
\textbf{Case 1c:} $\u\in\R^n \setminus (\U^\tau_y \cup \U^\tau_{n+1}) $ \\
We have $\pred^\CS_\tau(\u)=y'\neq y$. Also $p_{y'}\leq 1-p_y \leq \half$ and  $u_{(1)}=u_{y'} > u_{(2)} + \tau$.
\begin{eqnarray}
 \lefteqn{ \p^\top \bpsi^\CS(\u) - \p^\top \bpsi^\CS(\e_y) } \nonumber \\
 &\stackrel{(\ref{eqn:CS-lem-1})}{=}&
 \left( \sum_{i=1; i\neq y'}^n p_i \psi^\CS(i,\u) +   p_{y'} \psi^\CS(y',\u) \right) \nonumber \\
 &&\hspace{10em} -  2(1 - p_y)  \nonumber \\
 &\stackrel{(\ref{eqn:CS-lem-3}), (\ref{eqn:CS-lem-4})}{\geq}&
 (1-2p_{y'}) (u_{y'} - u_{(2)})  -1 +p_y \nonumber \\
 &\geq&
 2\tau(p_y-p_{y'})  \hspace{4em}\text{(From Case 1c)} \label{eqn:CS-case-1c-psi}
\end{eqnarray}
We also have that
\begin{equation}
\label{eqn:CS-case-1c-ell}
\p^\top \bell_{\pred^\CS_\tau(\u)} - \min_t \p^\top \bell_t = \p^\top \bell_{y'} - \p^\top \bell_y = p_y - p_{y'}  ~~ 
\end{equation}

From Equations (\ref{eqn:CS-case-1c-psi}) and (\ref{eqn:CS-case-1c-ell}) we have
\begin{eqnarray}
\lefteqn{ \p^\top \bpsi^\CS(\u) - \inf_{u'\in\R^n} \p^\top \bpsi^\CS(\u')} \nonumber \\ 
&\geq& 2  \tau \big( \p^\top \bell_{\pred^\CS_\tau(\u)} - \min_t \p^\top \bell_t \big) \label{eqn:CS-case1c}
\end{eqnarray}
\textbf{Case 2:} $p_{y'}<\half$ for all $y'\in[n]$

We have that $ n+1 \in \argmin_t \p^\top \bell _t $ 

\textbf{Case 2a:} $\u\in \U^\tau_{n+1}$ (or $\pred^\CS_\tau(\u)=n+1$) 

The RHS of equation (\ref{eqn:CS-excess-risk}) is zero, and hence becomes trivial. 

\textbf{Case 2b:} $\u\in \R^n \setminus \U^\tau_{n+1}$ (or $\pred^\CS_\tau(\u) \neq n+1$)

Let $\pred^\CS_\tau(\u)=\argmax_{i} u_{i}=y$. We have that $u_{(1)}= u_y > u_{(2)} + \tau$ and $p_y<\half$.
\begin{eqnarray}
\lefteqn{\p^\top \bpsi^\CS(\u) - \p^\top \bpsi^\CS(\0) } \nonumber \\
&\stackrel{(\ref{eqn:CS-lem-2})}{=}&
\left( \sum_{i=1; i\neq y}^n p_i \psi^\CS(i,\u) +  p_y \psi^\CS(y,\u)  \right)-  1  \nonumber \\
&\stackrel{(\ref{eqn:CS-lem-3}), (\ref{eqn:CS-lem-4})}{\geq}&
(1-2p_y) (u_{(1)} - u_{(2)}) \nonumber \\
&\geq&
(1-2p_y) (\tau) \hspace{3em}\text{(From Case 2b)} \label{eqn:CS-case-2b-psi}
\end{eqnarray}
We also have that
\begin{equation}
\label{eqn:CS-case-2b-ell}
\p^\top \bell_{\pred^\CS_\tau(\u)} - \min_t \p^\top \bell_t = \p^\top \bell_{y} - \p^\top \bell_{n+1}=   \half - p_y
\end{equation}
From Equations (\ref{eqn:CS-case-2b-psi}) and (\ref{eqn:CS-case-2b-ell}) we have
\begin{eqnarray}
\lefteqn{ \p^\top \bpsi^\CS(\u) - \inf_{u'\in\R^n} \p^\top \bpsi^\CS(\u') } \nonumber \\ 
&\geq& 2  \tau \big( \p^\top \bell_{\pred^\CS_\tau(\u)} - \min_t \p^\top \bell_t \big) \label{eqn:CS-case2b}
\end{eqnarray}
Equation (\ref{eqn:CS-excess-risk}), and hence the Theorem, follows from Equations (\ref{eqn:CS-case1b}), (\ref{eqn:CS-case1c}) and (\ref{eqn:CS-case2b}).
\end{proof}
 

\section{Proof of Excess Risk Bounds for the One vs All Hinge Surrogate}
\label{sec:app-B}
Define the sets $\U_1,\ldots,\U_{n+1}$ such that $\U_i$ is the set of vectors $\u$ in $\R^n$, for which $\pred^\OVA_\tau(\u)=i$
\begin{eqnarray*}
 \U^\tau_y \hspace{-0.5em}&=&\hspace{-0.5em} \{\u\in\R^n: u_y  >  \tau, y = \argmax_{i\in [n]} u_i \}, ~ y\in[n]\\
 \U^\tau_{n+1} \hspace{-0.5em}&=&\hspace{-0.5em} \{\u\in\R^n:  u_j  \leq  \tau \text{ for all }j\in[n] \}.
\end{eqnarray*}


The following lemma gives some crucial, but straightforward to prove, (in)equalities satisfied by the OVA hinge surrogate.
\begin{lem}
 \label{lem:OVA-identitites}
 \begin{eqnarray}
 \forall y\in[n],  \forall \p\in\Delta_n \hspace{-1em}&,&\hspace{-1em} \forall \u\in\R^n   \nonumber \\
  \p^\top \bpsi^\OVA(2\cdot \e_y - \1) &=& 4(1 - p_y)   \label{eqn:OVA-lem-1}\\
  \p^\top \bpsi^\OVA(-\1) &=& 2   \label{eqn:OVA-lem-2} \\
  \psi^\OVA(y,\u) &\geq&  \sum_{j\in[n]} u_j  -  2u_y + n   \label{eqn:OVA-lem-3} 
\end{eqnarray}
where $\e_y$ is the vector in $\R^n$ with $1$ in the $y^{th}$ position and $0$ everywhere else.
\end{lem}

The part of Theorem \ref{thm:CS-OVA-abstain-excess-risk} proved here is restated below.

\begin{thm*}
Let $n\in\N$ and $\tau\in(0,1)$. Then for all $f:\X\>\R^n$
\begin{eqnarray*} 
\lefteqn{\er_D^{\ell }[\pred^\OVA_\tau \circ \f] - \er_D^{\ell ,*} } \\
&\leq&  
\frac{1}{2 (1-|\tau|)}\left(\er_D^{\psi^\OVA}[\f] - \er_D^{\psi^\OVA,*}\right)
\end{eqnarray*}
\end{thm*}
\begin{proof}
We will show that $\forall \p\in\Delta_n$ and all $\u\in[-1,1]^d$
\begin{eqnarray}
 \lefteqn{\p^\top \bpsi^\OVA(\u) - \inf_{u'\in\R^n} \p^\top \bpsi^\OVA(\u')} \nonumber \\
 &\geq&
 2 (1-|\tau|) \big( \p^\top \bell_{\pred^\OVA_\tau(\u)} - \min_t \p^\top \bell_t \big)~~~~~~
 \label{eqn:OVA-excess-risk}
\end{eqnarray}
\vspace{-0.1em}
the Theorem simply follows from the observation that for all $\u\in\R^n$ clipping the components of $\u$ to $[-1,1]$ does not increase $\psi^\OVA(y,\u)$ for any $y$, and by linearity of expectation. 

\textbf{Case 1:} $p_y\geq \half$ for some $y\in[n]$.
 
We have that $y \in \argmin_t \p^\top \bell _t $.

\textbf{Case 1a:} $\u\in [-1,1]^n \cap  \U^\tau_y$ 

The RHS of equation (\ref{eqn:OVA-excess-risk}) is zero, and hence becomes trivial. 

\textbf{Case 1b:} $\u\in [-1,1]^n \cap \U^\tau_{n+1}$ 

We have that $ \max_j u_j \leq \tau$.
\begin{eqnarray}
 \lefteqn{\p^\top \bpsi^\OVA(\u)} \nonumber \\ 
 &\stackrel{\ref{eqn:OVA-lem-3}}{\geq}&
 \sum_{i=1}^n (1-2p_i) u_i + n     \nonumber \\
 &\geq&
 \sum_{i\in[n]\setminus\{y\}} (1-2p_i) u_i + (2p_y-1)(-\tau)  +n  \nonumber \\
 &\geq&
 \sum_{i\in[n]} (2p_i-1) + (2p_y-1)(-\tau-1) +n  \nonumber \\
 &=&
 (2p_y-1)(-\tau-1) + 2 \nonumber
\end{eqnarray}
And hence we have
\begin{eqnarray}
 \lefteqn{\p^\top \bpsi^\OVA(\u) - \p^\top \bpsi^\OVA(2\cdot \e_y - 1)} \nonumber \\ 
 &\stackrel{\ref{eqn:OVA-lem-1}}{\geq}&
 (2p_y-1)(-\tau-1) + 2 - 4(1-p_y) \nonumber \\
 &=&
 (2p_y-1)(1-\tau) \label{eqn:OVA-case-1b-psi}
\end{eqnarray}

We also have
\begin{equation}
\label{eqn:OVA-case-1b-ell}
 \p^\top \bell_{\pred^\OVA_\tau(\u)} - \min_t \p^\top \bell_t = \p^\top \bell_{n+1} - \p^\top \bell_y = p_y - \half
\end{equation}

From Equations (\ref{eqn:OVA-case-1b-psi}) and (\ref{eqn:OVA-case-1b-ell}) we have for all $\u\in [-1,1]^n \cap \U^\tau_{n+1}$ 
\begin{eqnarray}
\lefteqn{\p^\top \bpsi^\OVA(\u) - \inf_{u'\in\R^n} \p^\top \bpsi^\OVA(\u')} \nonumber \\
&\geq&  2  (1-\tau) \big( \p^\top \bell_{\pred^\OVA_\tau(\u)} - \min_t \p^\top \bell_t \big) \label{eqn:OVA-case1b}
\end{eqnarray}
\textbf{Case 1c:} $\u\in[-1,1]^n \setminus (\U^\tau_y \cup \U^\tau_{n+1}) $ 

We have $\pred^\OVA_\tau(\u)=y'\neq y$. Also $p_{y'}\leq \frac{1}{2}$ ;  $u_{y'} > \tau$ and $u_{y'}\geq u_y$.
\begin{eqnarray}
 \lefteqn{\p^\top \bpsi^\OVA(\u) } \nonumber  \\ 
 &\stackrel{\ref{eqn:OVA-lem-3}}{\geq}&
 \left( \sum_{i=1}^n (1-2p_i) u_i + n   \right) \nonumber \\
 &\geq&
 \left( \sum_{i\in[n]\setminus\{y'\}} (1-2p_i) u_i + (1 - 2p_{y'})(\tau)  +n  \right)  \hspace{2em} \nonumber \\
 &\geq&
 \left( \sum_{i\in[n]} (2p_i-1) + (1 - 2p_{y'})(\tau + 1) +n \right)  \;.\nonumber
 \end{eqnarray}
 And hence we have,
 \begin{eqnarray}
 \lefteqn{\p^\top \bpsi^\OVA(\u) - \p^\top \bpsi^\OVA(2\cdot \e_y - 1)} \nonumber \\ 
 &\stackrel{\ref{eqn:OVA-lem-1}}{\geq}&
 2+ (1-2p_{y'})(\tau+1) -4+4p_y \nonumber \\
 &=&
 (1-2p_{y'})(\tau+1) +2(2p_y-1) \nonumber \\
 & \geq &
 (1-2p_{y'})(\tau+1) +(1+\tau)\cdot(2p_y-1) \nonumber \\
 &=&
 2(1+\tau)(p_y - p_{y'}) \;.\label{eqn:OVA-case-1c-psi}
\end{eqnarray}
We also have that
\begin{equation}
\label{eqn:OVA-case-1c-ell}
\p^\top \bell_{\pred^\OVA_\tau(\u)} - \min_t \p^\top \bell_t = \p^\top \bell_{y'} - \p^\top \bell_y = p_y - p_{y'}   
\end{equation}

From Equations (\ref{eqn:OVA-case-1c-psi}) and (\ref{eqn:OVA-case-1c-ell}) we have for all $\u\in[-1,1]^n \setminus (\U^\tau_y \cup \U^\tau_{n+1})$ 
\begin{eqnarray}
\lefteqn{\p^\top \bpsi^\OVA(\u) - \inf_{u'\in\R^n} \p^\top \bpsi^\OVA(\u')} \nonumber  \\
&\geq& 2 (1+ \tau )\big( \p^\top \bell_{\pred^\OVA_\tau(\u)} - \min_t \p^\top \bell_t \big) \label{eqn:OVA-case1c}
\end{eqnarray}

\textbf{Case 2:} $p_{y'}<\half$ for all $y'\in[n]$

We have that $ n+1 \in \argmin_t \p^\top \bell _t $ 

\textbf{Case 2a:} $\u\in \U^\tau_{n+1}$ 

The RHS of equation (\ref{eqn:CS-excess-risk}) is zero, and hence becomes trivial. 

\textbf{Case 2b:} $\u\in [-1,1]^n \setminus \U^\tau_{n+1}$ 

Let $\pred^\OVA_\tau(\u)=\argmax_{i} u_{i}=y$. We have that $u_y \geq  \tau$ and $p_y<\half$.
\begin{eqnarray}
\lefteqn{\p^\top \bpsi^\OVA(\u) - \p^\top \bpsi^\OVA(-\1) } \nonumber \\
&\stackrel{(\ref{eqn:OVA-lem-2})}{=}&
\left( \p^\top \bpsi^\OVA(\u) \right)-  2  \nonumber \\
&\stackrel{(\ref{eqn:OVA-lem-3})}{\geq}&
\left( \sum_{i=1}^n (1-2p_i) u_i + n \right) -2 \nonumber \\
&\geq&
\left( \sum_{i\in[n]\setminus\{y\}} (1-2p_i) u_i + (1-2p_y)(\tau)  +n  \right) -  2   \nonumber \\
&\geq&
\left( \sum_{i\in[n]} (2p_i-1)  + (1-2p_y)(\tau+1)  +n  \right) -  2 \nonumber \\
&=&
(1-2p_y)(\tau+1) \label{eqn:OVA-case-2b-psi}
\end{eqnarray}

We also have that
\begin{equation}
\label{eqn:OVA-case-2b-ell}
\p^\top \bell_{\pred^\OVA_\tau(\u)} - \min_t \p^\top \bell_t = \p^\top \bell_{y} - \p^\top \bell_{n+1} = \half - p_y
\end{equation}
From Equations (\ref{eqn:OVA-case-2b-psi}) and (\ref{eqn:OVA-case-2b-ell}) we have for all $\u\in [-1,1]^n \setminus \U^\tau_{n+1}$ 
\begin{eqnarray}
\lefteqn{\p^\top \bpsi^\OVA(\u) - \inf_{u'\in\R^n} \p^\top \bpsi^\OVA(\u')} \nonumber \\
&\geq& 2  (1+\tau) \big( \p^\top \bell_{\pred^\OVA_\tau(\u)} - \min_t \p^\top \bell_t \big) \label{eqn:OVA-case2b}
\end{eqnarray}
Equation (\ref{eqn:OVA-excess-risk}), and hence the Theorem, follows from Equations (\ref{eqn:OVA-case1b}), (\ref{eqn:OVA-case1c}) and (\ref{eqn:OVA-case2b}).
\end{proof}

\section{Proof of Excess Risk Bounds for the BEP Surrogate}
\label{sec:app-C}
The following lemma gives some crucial, but straightforward to prove, (in)equalities satisfied by the BEP surrogate.
\begin{lem}
 \label{lem:BEP-identitites}
 \begin{eqnarray}
\forall y,y'\in[n],   \p\in\Delta_n ,  \u\in\hspace{-0.8em}&\R^n&\hspace{-0.8em}, y'\neq B^{-1}(\sign(-\u))   \nonumber \\
  \p^\top \bpsi^\BEP(-B(y)) &=& 2(1 - p_y) 		  \label{eqn:BEP-lem-1} \\
  \p^\top \bpsi^\BEP(\0) &=& 1  			  \label{eqn:BEP-lem-2} \\
  \psi^\BEP(B^{-1}(\sign(-\u)),\u) &\geq& - \min_j |u_j| + 1  \label{eqn:BEP-lem-3} \\
  \psi^\BEP(y',\u) &\geq& \min_j |u_j| + 1 		   \label{eqn:BEP-lem-4}
\end{eqnarray}
\end{lem}

\begin{thm*}
Let $n\in\N$ and $\tau\in(0,1)$. Let $n=2^d$. Then for all $f:\X\>\R^d$
\begin{eqnarray*}
\er_D^{\ell }[\pred^\BEP_\tau \circ \f] - \er_D^{\ell ,*} 
&\leq&  
\frac{\left(\er_D^{\psi^\BEP}[\f] - \er_D^{\psi^\BEP,*}\right) }{2 \min(\tau,1-\tau)}
\end{eqnarray*}
\end{thm*}
\begin{proof}
We will show that $\forall\p\in\Delta_n$ and all $\u\in\R^d$ 
\begin{eqnarray}
\lefteqn{\p^\top \bpsi^\BEP(\u) - \inf_{u'\in\R^d} \p^\top \bpsi^\BEP(\u') } \nonumber \\
&\geq& 2\min(\tau,1-\tau) (\p^\top \bell_{\pred^\BEP_\tau(\u)} - \min_t \p^\top \bell_t) ~~~~\label{eqn:BEP-excess-risk}
\end{eqnarray}
The theorem follows by linearity of expectation.

\textbf{Case 1: } $p_y \geq \half$ for some $y\in[n]$

We have that $y\in\argmin_{t}\p^\top \bell_t$

\textbf{Case 1a:} $\u\in \U^\tau_y $ (or $\pred^\BEP_\tau(\u)=y$) 

The RHS of equation (\ref{eqn:BEP-excess-risk}) is zero, and hence becomes trivial. 
 
\textbf{Case 1b:} $\u\in \U^\tau_{n+1}$ (or $\pred^\BEP_\tau(\u)=n+1$) 

Let $y'= B^{-1}(\sign(-\u))$. We have   $\min_j |u_j| \leq \tau$.
\begin{eqnarray}
 \lefteqn{\p^\top\bpsi^\BEP(\u) } \nonumber \\
 &=&
 p_{y' }\psi^\BEP(y',\u) + \sum_{i\in[n]\setminus \{y'\}} p_i \psi^\BEP(i,\u) \nonumber \\
 &\stackrel{(\ref{eqn:BEP-lem-3}),(\ref{eqn:BEP-lem-4})}{\geq}&
 p_{y' }(-\min_{j\in[d]} |u_j|)  + (1-p_{y'}) (\min_{j\in[d]} |u_j|) + 1 \nonumber \\
 &\geq&
 (2p_y -1 )(-\tau) + 1  \;. \nonumber
\end{eqnarray}
The last inequality in the above follows from the observation that if $y'\neq y$, then $p_{y'} < 1-p_y \leq \half$. We thus have
\begin{eqnarray}
 \lefteqn{\p^\top\bpsi^\BEP(\u) - \p^\top \bpsi^\BEP(-B(y)) } \nonumber \\
 &\stackrel{(\ref{eqn:BEP-lem-1})}{\geq}&
 (2p_y -1 )(1-\tau)  \;.\label{eqn:BEP-case-1b-psi}
\end{eqnarray}

We also have that
\begin{equation}
\label{eqn:BEP-case-1b-ell}
 \p^\top \bell_{\pred^\BEP_\tau(\u)} - \min_t \p^\top \bell_t = \p^\top \bell_{n+1} - \p^\top \bell_y= p_y - \half
\end{equation}
From Equations (\ref{eqn:BEP-case-1b-psi}) and (\ref{eqn:BEP-case-1b-ell}) we have that
\begin{eqnarray}\
 \lefteqn{\p^\top \bpsi^\BEP(\u) - \inf_{u'\in\R^d} \p^\top \bpsi^\BEP(\u')} \nonumber \\
 &\geq&  2(1-\tau) (\p^\top \bell_{\pred^\BEP_\tau(\u)} - \min_t \p^\top \bell_t) \label{eqn:BEP-case1b}
\end{eqnarray}
\textbf{Case 1c: $\u\in \R^d \setminus (\U^\tau_y \cup \U^\tau_{n+1})$} 

Let $B^{-1}(\sign(-\u))=\pred(\u)=y'$ for some $y'\neq y$. We have
$p_{y'}\leq 1-p_y \leq \half$, and  
$\min_j |u_j|>\tau$ and
\begin{eqnarray}
\lefteqn{\p^\top\bpsi^\BEP(\u) } \nonumber \\
&=& 
p_{y'} \psi^\BEP(y',\u) + \sum_{i=1;i\neq y'}^n p_i \psi^\BEP(i,\u) \nonumber\\ 
&\stackrel{(\ref{eqn:BEP-lem-3}),(\ref{eqn:BEP-lem-4})}{\geq}&
p_{y'} (-\min_j |u_j|) + (1-p_{y'}) (\min_j |u_j|) + 1\nonumber\\ 
&\geq&
\tau(1-2p_{y'}) + 1 \hspace{2.6em}\text{(From case 1c)} \nonumber
\end{eqnarray}
Hence we get
\begin{eqnarray}
 \p^\top\bpsi^\BEP(\u) - \p^\top\bpsi^\BEP(-B(y))  
 \stackrel{(\ref{eqn:BEP-lem-1})}{\geq}
2\tau(p_y-p_{y'})  \label{eqn:BEP-case-1c-psi} 
\end{eqnarray}

We also have that
\begin{equation}
\label{eqn:BEP-case-1c-ell}
\p^\top \bell_{\pred^\BEP_\tau(\u)} - \min_t \p^\top \bell_t = \p^\top \bell_{y'} - \p^\top \bell_y = p_y - p_{y'}   
\end{equation}
From Equations (\ref{eqn:BEP-case-1c-psi}) and (\ref{eqn:BEP-case-1c-ell}) we have that
\begin{eqnarray}
 \lefteqn{\p^\top \bpsi^\BEP(\u) - \inf_{u'\in\R^d} \p^\top \bpsi^\BEP(\u')} \nonumber \\
 &\geq& 2(\tau) (\p^\top \bell_{\pred^\BEP_\tau(\u)} - \min_t \p^\top \bell_t) \label{eqn:BEP-case1c}
\end{eqnarray}
\textbf{Case 2: } $p_y<\half$ for all $y\in[n]$
 
We have that $ n+1 \in \argmin_t \p^\top \bell_t $

\textbf{Case 2a: $\u\in \U^\tau_{n+1}$}

The RHS of equation (\ref{eqn:BEP-excess-risk}) is zero, and hence becomes trivial. 

\textbf{Case 2b:}  $\u\in \R^d \setminus \U^\tau_{n+1}$

Let 
$B^{-1}(\sign(-\u))=y'=\pred^\BEP_\tau(\u)$ for some $y'\in[n]$. We have $p_{y'}<\half$ and
$\min_j |u_j|>\tau$.
\begin{eqnarray}
\lefteqn{\p^\top\bpsi^\BEP(\u) - \p^\top \bpsi^\BEP(\0)} \nonumber \\ 
&\stackrel{(\ref{eqn:BEP-lem-2})}{=}& 
p_{y'} \psi^\BEP(y',\u) + \sum_{i=1;i\neq y'}^n p_i \psi^\BEP(i,\u) -  1 \nonumber \\ 
&\stackrel{(\ref{eqn:BEP-lem-3}),(\ref{eqn:BEP-lem-4})}{\geq}&
-p_{y'} \min_j |u_j| +  (1-p_{y'}) \min_j |u_j| \nonumber \\
&\geq&
(1-2p_{y'})\tau  \label{eqn:BEP-case-2b-psi}
\end{eqnarray}
We also have that
\begin{equation}
\label{eqn:BEP-case-2b-ell}
\p^\top \bell_{\pred^\BEP_\tau(\u)} - \min_t \p^\top \bell_t = \p^\top \bell_{y'} - \p^\top \bell_{n+1} =  \half - p_{y'}
\end{equation}
From Equations (\ref{eqn:BEP-case-2b-psi}) and (\ref{eqn:BEP-case-2b-ell}) we have that
\begin{eqnarray}
 \lefteqn{\p^\top \bpsi^\BEP(\u) - \inf_{u'\in\R^d} \p^\top \bpsi^\BEP(\u')} \nonumber \\
 &\geq& 2\tau (\p^\top \bell_{\pred^\BEP_\tau(\u)} - \min_t \p^\top \bell_t) \label{eqn:BEP-case2b}
\end{eqnarray}
Equation (\ref{eqn:BEP-excess-risk}), and hence the Theorem, follows from equations (\ref{eqn:BEP-case1b}), (\ref{eqn:BEP-case1c}) and (\ref{eqn:BEP-case2b}).
\end{proof}

\end{document}